\documentclass[final]{l4dc2026}

\title[Can Optimal Transport Improve Federated Inverse Reinforcement Learning?]{Can Optimal Transport Improve Federated Inverse \\ Reinforcement Learning?}
\usepackage{times}
\usepackage{booktabs}
\usepackage{graphicx}
\usepackage{subcaption}

\coltauthor{%
\Name{David Millard} \Email{djm3622@rit.edu}\\
\Name{Ali Baheri} \Email{akbeme@rit.edu}\\
\addr Department of Mechanical Engineering, Rochester Institute of Technology, Rochester, NY, USA
}

\DeclareMathOperator*{\argmin}{arg\,min}
\DeclareMathOperator*{\argmax}{arg\,max}
\DeclareMathOperator{\E}{\mathbb{E}}
\newcommand{\R}{\mathbb{R}}
\newtheorem{assumption}{Assumption}

\begin{document}

\maketitle

\begin{abstract}%
In robotics and multi-agent systems, fleets of autonomous agents often operate in subtly different environments while pursuing a common high-level objective. Directly pooling their data to learn a shared reward function is typically impractical due to differences in dynamics, privacy constraints, and limited communication bandwidth. This paper introduces an optimal transport–based approach to federated inverse reinforcement learning (IRL). Each client first performs lightweight Maximum Entropy IRL locally, adhering to its computational and privacy limitations. The resulting reward functions are then fused via a Wasserstein barycenter, which considers their underlying geometric structure. We further prove that this barycentric fusion yields a more faithful global reward estimate than conventional parameter averaging methods in federated learning. Overall, this work provides a principled and communication-efficient framework for deriving a shared reward that generalizes across heterogeneous agents and environments.
\end{abstract}


\section{Introduction}
With the growing demand for personalized robotics, it has become increasingly important to enable experience sharing that accelerates learning while preserving privacy. These robots often operate under expert policies in environments whose dynamics evolve over time—for instance, a household vacuum robot adapting to a new layout or a self-driving car navigating an unfamiliar city. Such shifts alter the underlying transition dynamics, yet the robot’s performance should remain robust to these changes. Moreover, within a network of robots, it is likely that another agent has already adapted to a similar environment. However, because each robot operates in a private domain, it is essential to prevent any information leakage about an individual client’s environment. Finally, many robotic platforms are designed primarily for inference or lightweight fine-tuning and therefore lack the computational resources to perform full-scale IRL to recover optimal reward parameters for their specific operating conditions. This setting motivates the problem of federated IRL, the ability to share and aggregate learned reward functions across multiple agents in a privacy-preserving manner.

Many recent works have begun to explore federated variants of inverse reinforcement learning. In many approaches, model fusion is achieved through direct parameter averaging~\cite{mcmahan2017communication}, without accounting for the geometric structure of the underlying reward landscape. While these averaging schemes are simple and communication-efficient, they often yield suboptimal or unstable global rewards when local parameters are poorly fitted, an expected outcome for clients with limited computational resources or insufficient expert demonstrations. In highly heterogeneous settings, such instability can propagate through aggregation, leading to inconsistent or even unsafe reward estimates that degrade overall system performance. Moreover, the resulting global reward may fail to preserve consistent behavioral semantics across clients operating under distinct dynamics. These limitations motivate the development of a geometry-aware fusion mechanism that produces faithful global rewards even under noisy, biased, or undertrained local estimates.

To address these challenges, we treat each locally learned reward in the discrete setting as what it parameterizes: a function over a shared state–action support. We explicitly evaluate this function on the common lattice, apply a shift-and-normalize step to obtain a probability distribution of reward mass, and then fuse clients by computing an entropically regularized Wasserstein barycenter over these normalized measures. This geometry-aware fusion preserves spatial structure, dampens underfit artifacts from weak clients, and yields a stable, semantically consistent global reward under heterogeneity. Finally, we project the fused measure back to parameters via least squares in the shared feature basis, enabling lightweight clients to contribute without sharing trajectories or destabilizing the model. Our contributions are as follows:
\begin{enumerate}
    \item We introduce a federated inverse reinforcement learning framework that fuses learned reward functions using an entropically regularized Wasserstein barycenter.
    \item We provide a theoretical analysis establishing stability and parameter-error bounds for the proposed barycentric aggregation, showing that it contracts toward the true reward function under bounded local estimation error.
    \item We demonstrate through empirical studies on discrete and continuous-control benchmarks that barycentric reward fusion consistently yields more stable and transferable global rewards than parameter averaging.
\end{enumerate}

\section{Related Work}\label{sec:related}
Learning reward functions across multiple agents with heterogeneous dynamics has been explored from several perspectives. In federated and distributed IRL, prior work has developed decentralized optimization schemes and convergence guarantees for recovering rewards without centralizing expert trajectories~\cite{gong2023federated,banerjee2021identity,hassan2024enhancing}. However, these approaches typically retain the standard parameter-averaging server update from conventional federated learning, focusing primarily on modifying local optimization steps rather than the aggregation mechanism itself. Beyond purely federated settings, multi-task and transfer-style IRL methods link related tasks to share statistical strength, for example, through maximum causal entropy formulations, teacher–student architectures, and group learning \cite{gleave2018multi,noothigattu2021inverse,liu2022distributed}. In contrast, our formulation preserves the privacy and communication constraints central to federated learning while explicitly addressing heterogeneity in transition dynamics and demonstration distributions, as well as hardware limitations that can lead to underfitted local reward parameters.

Classical IRL methods include early formulations of the inverse problem~\cite{ng2000algorithms}, apprenticeship learning by feature expectation matching~\cite{abbeel2004apprenticeship}, maximum entropy (MaxEnt) IRL~\cite{ziebart2008maximum}, and adversarial variants such as GAIL~\cite{ho2016generative}. We adopt a MaxEnt local learner for its robustness and well-studied gradients \cite{snoswell2020revisiting}. Moreover, MaxEnt IRL integrates with the linear function approximation, which we employ during experimentation to maintain meaningful support while keeping computational complexity manageable~\cite{song2025survey}. Although our framework is designed to be lightweight, we explicitly assume a limited computational budget on the client side, consistent with prior work highlighting the resource constraints of embedded and mobile robotic platforms~\cite{neuman2022tiny,imteaj2020fedar}. In the broader federated learning literature, FedAvg remains the canonical baseline~\cite{mcmahan2017communication}, and, to the best of our knowledge, no prior work has proposed an alternate aggregation technique when federating IRL. 

Optimal transport provides a geometry-aware framework for comparing and aggregating distributions. Entropic regularization enables efficient computation through the Sinkhorn algorithm~\cite{cuturi2013sinkhorn}, and the Wasserstein barycenter defines a principled notion of averaging over probability measures~\cite{agueh2011barycenters}. Recent work has applied optimal transport to reinforcement learning, using it to analyze reward ambiguity in IRL~\cite{baheri2023understanding} and to develop consensus or policy update mechanisms in multi-agent RL based on Wasserstein distance~\cite{pereira2025heterogeneous,baheri2025wassersteinbarycenterconsensuscooperativemultiagent,shahrooei2025wasserstein,salgarkar2025distance}. In this work, we compute a Wasserstein barycenter over locally learned reward functions to obtain a fused reward that remains interpretable, privacy-preserving, and robust to training inconsistencies.

\section{Preliminaries \& Problem Formulation}\label{sec:prelim}

This section presents the components of our method and defines the problem domain it addresses.

\subsection{Inverse Reinforcement Learning}

IRL seeks to recover the underlying reward function that explains observed expert behavior in a Markov Decision Process (MDP). Given an MDP $\mathcal{M} = (\mathcal{S}, \mathcal{A}, P, \gamma)$ with an unknown reward function $R^*: \mathcal{S} \times \mathcal{A} \to \mathbb{R}$ and expert demonstrations $\mathcal{D}^E = \{\tau_j\}_{j=1}^{N}$ where each trajectory $\tau_j = (s_0, a_0, \dots, s_T, a_T)$ represents a sequence of state–action pairs, the goal is to infer a reward function $\hat{R}$ such that the optimal policy under $\hat{R}$ reproduces the expert’s behavior. Formally, IRL seeks $\hat{R}$ satisfying
\begin{align}
\pi_{\hat{R}}^* = \argmax_{\pi} \; \mathbb{E}_{\pi,P}\!\left[\sum_{t=0}^T \gamma^t \hat{R}(s_t,a_t)\right],
\quad \text{with} \quad
\pi_{\hat{R}}^* \approx \pi_{R^*}^*.
\end{align}
In our work, we seek to find a linear function approximation such that $\hat{R}_\theta(s_t,a_t) = \theta^\top \phi(s_t,a_t),$ where $\phi(s,a)$ denotes the feature vector and $\theta$ the reward parameters. MaxEnt IRL models the trajectory distribution as
\begin{align}
p(\tau \,|\, \theta) = \frac{1}{Z(\theta)} 
\exp\!\left(\sum_{t=0}^T \gamma^t \theta^\top \phi(s_t,a_t)\right),
\end{align}
where $Z(\theta)$ represents the partition function ensuring normalization. The objective maximizes the likelihood of expert trajectories under this distribution while maintaining maximum entropy. 

\subsection{Wasserstein Barycenter}
Optimal transport provides a principled mathematical framework for measuring dissimilarities between probability distributions by considering the minimal transport cost required to move one distribution onto another. Given measures $\mu$ and $\nu$ on a metric space $(\mathcal{X}, d)$, the 2-Wasserstein distance is defined as
\begin{align}
W_2^2(\mu, \nu) = \inf_{\gamma \in \Pi(\mu, \nu)} \int_{\mathcal{X} \times \mathcal{X}} d(x,y)^2 \, d\gamma(x,y),
\end{align}
where $\Pi(\mu, \nu)$ is the set of all couplings whose marginals are $\mu$ and $\nu$. Intuitively, $W_2$ represents the minimum amount of work that must be done to transform one probability distribution into another, capturing not only pointwise differences but also their geometric relationships. Building on this, the Wasserstein barycenter generalizes the notion of an arithmetic mean to the space of probability measures endowed with the Wasserstein metric. Given a collection of distributions $\{\mu_i\}_{i=1}^N$ and nonnegative weights $\{\alpha_i\}_{i=1}^N$ summing to one, their barycenter $\bar{\mu}$ is defined as the minimizer
\begin{align}
\bar{\mu} = \argmin_{\mu \in \mathcal{P}_2(\mathcal{X})} \sum_{i=1}^N \alpha_i W_2^2(\mu, \mu_i),
\end{align}
where $\mathcal{P}_2(\mathcal{X})$ denotes the space of probability measures with finite second moments. The barycenter thus represents the most geometrically consistent “average” of the input distributions, one that lies in a central position with respect to the Wasserstein distance.

\subsection{Problem Formulation}
We consider the problem of horizontal federated reinforcement learning \cite{qi2021federated}. We begin with $N$ clients, where each client $i$ may correspond to a single agent or a set of cooperating agents. Let $\mathcal{M}_i = (\mathcal{S}, \mathcal{A}, P_i, \gamma)$ denote the MDP for client $i$, where $\mathcal{S}$ is the shared state space, $\mathcal{A}$ is the shared action space, $P_i: \mathcal{S} \times \mathcal{A} \times \mathcal{S} \to [0,1]$ represents the client-specific transition dynamics, and $\gamma \in (0,1)$ is the discount factor. 
\begin{assumption}[Shared finite support and metric]\label{ass:shared-support}
All clients share a common finite state--action space $X = S \times A$ with $|X| = n$, equipped with a ground metric $d$ of diameter $D$ and minimum nonzero spacing $\delta>0$. In continuous-control experiments, this finite space corresponds to a discretized lattice over the joint state--action domain used to construct the transport cost.
\end{assumption}
Next, each client $F_i$ has access to a local dataset of expert demonstrations 
\begin{align}
    \mathcal{D}_i^E = \{\tau_j^{(i)}\}_{j=1}^{N_i}, \quad \text{where} \quad 
\tau_j^{(i)} = (s_0, a_0, s_1, a_1, \ldots, s_T, a_T)
\end{align}
represents a trajectory of state–action pairs generated by an expert operating under an unknown reward function 
$R^*: \mathcal{S} \times \mathcal{A} \to \mathbb{R}$.
The fundamental challenge in federated IRL  is to recover a shared reward function that explains expert behavior across all clients while accounting for heterogeneity in both dynamics and demonstrations, without disclosing any expert data or behavioral information to other clients or the central server. Furthermore, the recovered reward function $\hat{R}$ should not only explain the expert demonstrations observed in each local MDP $\mathcal{M}_i$, but also generalize across both seen and unseen environments. Formally, $\hat{R}$ should satisfy
$
\hat{R} \approx R^* \quad \text{on} \quad \bigcup_{i=1}^N \mathcal{M}_i,
$
and exhibit transferability such that, for any new MDP 
$\mathcal{M}_{\text{new}} = (\mathcal{S}, \mathcal{A}, P_{\text{new}}, \gamma)$ 
with transition dynamics $P_{\text{new}}$ drawn from the same distributional family as $\{P_i\}_{i=1}^N$, the induced optimal policy
\begin{align}
\pi_{\hat{R}}^* 
= \argmax_{\pi} \; 
\E_{\pi, P_{\text{new}}}\!\left[\sum_{t=0}^{T} \gamma^t \hat{R}(s_t, a_t)\right]
\end{align}
closely matches the expert policy $\pi_{R^*}^*$ under the true reward $R^*$. 
\begin{assumption}[Feature map and conditioning]\label{ass:features}
The reward is represented in a shared feature basis $\Phi \in \mathbb{R}^{n\times d}$ of full column rank. Let $\sigma_{\min}(\Phi)$ denote its smallest singular value and define $\kappa_\Phi := \sigma_{\min}(\Phi)^{-1} < \infty$.
\end{assumption}

\section{Methodology}\label{sec:method}

In this section, we present our method. Unlike standard federated learning, which alternates local updates and global aggregation over multiple rounds, we perform a single aggregation step where locally learned rewards are fused once using the proposed barycentric approach.

\subsection{Client-side Learning}

Each client independently solves a local inverse reinforcement learning problem using the MaxEnt IRL framework. For client $F_i$, we assume a linear reward parameterization $R_i(s,a) = \theta_i^T \phi(s,a)$, where $\phi: \mathcal{S} \times \mathcal{A} \to \R^d$ is a feature mapping and $\theta_i \in \R^d$ are the reward parameters. The MaxEnt IRL objective for agent $i$ is:
\begin{equation}
\mathcal{L}_i(\theta_i) = \E_{\tau \sim \mathcal{D}_i^E}\left[\sum_{t=0}^T \gamma^t \theta_i^T \phi(s_t, a_t)\right] - \log Z_i(\theta_i) - \frac{\lambda}{2} \|\theta_i\|_2^2
\end{equation}
where $Z_i(\theta_i)$ is the partition function. The gradient of the objective with respect to $\theta_i$ is:
$
\nabla_{\theta_i} \mathcal{L}_i = \mu_{\mathcal{D}_i^E} - \mu_{\pi_{\theta_i}} - \lambda \theta_i
$
where $\mu_{\mathcal{D}_i^E}$ represents the empirical feature expectations and $\mu_{\pi_{\theta_i}}$ denotes the expected feature counts under the policy induced by the current reward parameters. 

\subsection{Aggregation}

After local learning, each client $F_i$ has parameters $\theta_i^*$ inducing a reward function $R_i(s,a) = (\theta_i^*)^T \phi(s,a)$. These parameters are then sent to a central server. This server assumes access to a probe set (a deterministic environment that is similar to the clients). We then aggregate these heterogeneous reward functions into a consensus reward using optimal transport theory. We treat each reward function as a signed measure and seek the Wasserstein barycenter:
\begin{equation}
R^* = \argmin_{R} \sum_{i=1}^K \alpha_i W_2^2(R, R_i)
\end{equation}
where $\alpha_i$ are aggregation weights and $W_2$ is the 2-Wasserstein distance. We make this tractable by converting each reward function to a probability measure via a positive shift and normalization. Furthermore, the direct computation of the Wasserstein barycenter is computationally prohibitive. We employ entropic regularization, leading to a smoothed problem that can be solved efficiently via the Sinkhorn-Knopp algorithm. The Wasserstein barycenter is then computed via alternating optimization with momentum updates for faster convergence.

\subsection{Parameter Recovery}

After computing the barycenter reward function $R^*$, we recover the corresponding parameters $\theta^*$ by solving a least-squares problem:
\begin{equation}
\theta^* = \argmin_{\theta} \sum_{s \in \mathcal{S}} \|R^*(s, a) - \theta^T \phi(s, a)\|_2^2
\end{equation}
This has the closed-form solution $\theta^* = (\Phi^T \Phi)^{-1} \Phi^T r^*$, where $\Phi$ is the feature matrix and $r^*$ is the vectorized barycenter reward.

\section{Theoretical Results}\label{sec:theory}
We now analyze the effect of Wasserstein barycentric fusion on both the learned reward and the downstream control performance. 
\begin{theorem}[Stability and parameter-error bound]\label{thm:fed-irl-ot-main}
Suppose Assumptions~\ref{ass:shared-support} and~\ref{ass:features} hold.
Each agent $i\in\{1,\dots,K\}$ runs local MaxEnt IRL and outputs a reward vector $\hat r_i\in\mathbb{R}^n$.
Fix a shift $\sigma>0$ so that all entries of $\hat r_i+\sigma\mathbf{1}$ and of a shared population reward $r^\star$ are positive, and define
\begin{align}
T_\sigma(r)\;=\;\frac{r+\sigma\mathbf{1}}{\mathbf{1}^\top(r+\sigma\mathbf{1})}\in\Delta_n,\quad
&p_i:=T_\sigma(\hat r_i),\\ &p^\star:=T_\sigma(r^\star),\\
&Z(r):=\mathbf{1}^\top(r+\sigma\mathbf{1}),
\end{align}
with $Z_{\min}:=\min\,\!\bigl(\min_i Z(\hat r_i),\,Z(r^\star)\bigr)$.
Let $\bar p$ be the $2$-Wasserstein barycenter of $\{p_i\}_{i=1}^K$ with weights $\alpha\in\Delta_K$,
and map back to a reward via $\bar r_Z:=Z\,\bar p-\sigma\mathbf{1}$ for a chosen scalar $Z>0$.
Assume local IRL accuracy
$
\|\hat r_i - r^\star\|_1 \le \varepsilon_i \quad (i=1,\dots,K).
$
Then:
\begin{align}\label{eq:T1}
W_2(\bar p,\,p^\star)
&\le \frac{2D}{\sqrt{Z_{\min}}}\Bigl(\textstyle\sum_{i=1}^K \alpha_i \varepsilon_i\Bigr)^{1/2}, 
\end{align}
and if we choose $Z=Z(r^\star)$ so that $r^\star=Z(r^\star)\,p^\star-\sigma\mathbf{1}$, then
\begin{align} \label{eq:T2}
\|\bar r_{Z(r^\star)} - r^\star\|_2
&\le \frac{2\sqrt{2}\,D\,Z(r^\star)}{\delta\sqrt{Z_{\min}}}\Bigl(\textstyle\sum_{i=1}^K \alpha_i \varepsilon_i\Bigr)^{1/2}. 
\end{align}
Moreover, with $\theta^\star=\arg\min_\theta\|\Phi\theta-r^\star\|_2$ and
$\hat\theta=(\Phi^\top\Phi)^{-1}\Phi^\top \bar r_{Z(r^\star)}$, the parameter error satisfies
\begin{align}\label{eq:T3}
\|\hat\theta-\theta^\star\|_2
\;\le\;
\kappa_\Phi\;\frac{2\sqrt{2}\,D\,Z(r^\star)}{\delta\sqrt{Z_{\min}}}\Bigl(\textstyle\sum_{i=1}^K \alpha_i \varepsilon_i\Bigr)^{1/2},
\end{align}
where $\kappa_\Phi$ is as in Assumption~\ref{ass:features}.
\end{theorem}
\begin{proof}[Proof sketch]
(1) \emph{Shift-normalize is Lipschitz in $\ell_1$.}
For $p=T_\sigma(r)$, $q=T_\sigma(s)$, one shows
$\|p-q\|_1 \le (2/Z_{\min})\|r-s\|_1$ by a quotient-rule bound.
(2) \emph{From $\ell_1$ to $W_2$.}
On a finite metric space, $W_2^2(\mu,\nu)\le D^2\,\mathrm{TV}(\mu,\nu)=\tfrac{D^2}{2}\|\mu-\nu\|_1$.
Together with (1) this gives $W_2^2(p_i,p^\star)\le (D^2/Z_{\min})\,\varepsilon_i$.
(3) \emph{Barycenter contraction.}
By triangle inequality and the optimality of the Wasserstein barycenter (cf. §II-F),
$W_2^2(\bar p,p^\star)\le 4\sum_i \alpha_i W_2^2(p_i,p^\star)$, yielding \eqref{eq:T1}.
(4) \emph{Back to rewards and parameters.}
With $Z=Z(r^\star)$, $\bar r_{Z(r^\star)}-r^\star = Z(r^\star)(\bar p-p^\star)$ and
$\|\bar p-p^\star\|_2 \le \sqrt{2}\,W_2(\bar p,p^\star)/\delta$, which implies \eqref{eq:T2}.
Finally $\|\hat\theta-\theta^\star\|_2\le \kappa_\Phi\,\|\bar r_{Z(r^\star)}-r^\star\|_2$, giving \eqref{eq:T3}.
\end{proof}
Theorem~\ref{thm:fed-irl-ot-main} implies two key properties. First, aggregation does not amplify noise: if each local reward is $\varepsilon_i$-close to $r^\star$, then the fused reward $\bar r$ and its recovered parameter vector $\hat\theta$ are $O((\sum_i \alpha_i \varepsilon_i)^{1/2})$-close to $r^\star$ and $\theta^\star$, respectively. Second, the dependence is only through the weighted average of local errors, so a single weak client cannot dominate the fused model.We now connect this stability in reward space to control performance. The next result bounds, for each client, the suboptimality gap in discounted return when executing the policy induced by the fused reward $\bar r$ instead of the (unknown) optimal policy for $r^\star$.
\begin{theorem}[Policy--performance bound for barycentric fusion]\label{thm:perf}
Fix a client $i\in\{1,\dots,K\}$ with MDP $M_i=(\mathcal S,\mathcal A,P_i,\gamma)$ and let $X=\mathcal S\times\mathcal A$ be finite, $|X|=n$. For any policy $\pi$, let
\begin{align}
\rho_i^\pi(x)\;:=\;(1-\gamma)\sum_{t=0}^\infty \gamma^t\,\mathbb P_{P_i,\pi}\big((s_t,a_t)=x\big),\qquad x\in X,
\end{align}
be the normalized $\gamma$-discounted occupancy measure (so $\|\rho_i^\pi\|_1=1$). For any reward vector $r\in\mathbb R^n$, define the (discounted) return of $\pi$ on client $i$ by
$
J_i(\pi;r)\;:=\;\frac{1}{1-\gamma}\,\langle \rho_i^\pi,\, r\rangle.
$
Let $\pi_r\in\arg\max_\pi J_i(\pi;r)$ be an optimal policy for reward $r$ on $M_i$. Then for any $r,s\in\mathbb R^n$,
\begin{equation}\label{eq:perf-core}
0\;\le\; J_i(\pi_r; r) - J_i(\pi_s; r)
\;\le\; \frac{2}{1-\gamma}\,\|r-s\|_\infty.
\end{equation}
In particular, with $r_\star$ denoting the population reward and $\bar r:=\bar r_{Z(r_\star)}$ the fused reward produced by the barycentric mapping and back‑projection of Theorem~1, we have for every client $i$,
\begin{equation}\label{eq:perf-rbar}
0\;\le\; J_i(\pi_{r_\star}; r_\star) - J_i(\pi_{\bar r}; r_\star)
\;\le\; \frac{2}{1-\gamma}\,\|r_\star-\bar r\|_\infty
\;\le\; \frac{2}{1-\gamma}\,\|r_\star-\bar r\|_2.
\end{equation}
Combining \eqref{eq:perf-rbar} with the reward-space bound of Theorem~1 yields the explicit end‑to‑end guarantee
\begin{equation}\label{eq:perf-explicit}
J_i(\pi_{r_\star}; r_\star) - J_i(\pi_{\bar r}; r_\star)
\;\le\; \frac{4\sqrt{2}\,D\,Z(r_\star)}{(1-\gamma)\,\delta\,\sqrt{Z_{\min}}}
\Bigg(\sum_{j=1}^K \alpha_j\,\varepsilon_j\Bigg)^{\!1/2}.
\end{equation}
\end{theorem}
\begin{proof}
Because $\pi_r$ maximizes $J_i(\pi;r)$, the left inequality in \eqref{eq:perf-core} is immediate. For the upper bound,
\begin{align}
\begin{aligned}
J_i(\pi_r; r) - J_i(\pi_s; r)
&= \frac{1}{1-\gamma}\,\big\langle \rho_i^{\pi_r}-\rho_i^{\pi_s},\, r \big\rangle \\
&= \frac{1}{1-\gamma}\,\big\langle \rho_i^{\pi_r}-\rho_i^{\pi_s},\, r-s \big\rangle
\;+\; \frac{1}{1-\gamma}\,\Big(\langle \rho_i^{\pi_r}, s\rangle - \langle \rho_i^{\pi_s}, s\rangle\Big).
\end{aligned}
\end{align}
The second term is $\le 0$ because $\pi_s$ maximizes $J_i(\pi;s)$. Applying Hölder's inequality yields
\begin{align}
J_i(\pi_r; r) - J_i(\pi_s; r)
\;\le\; \frac{1}{1-\gamma}\,\|\rho_i^{\pi_r}-\rho_i^{\pi_s}\|_1\,\|r-s\|_\infty.
\end{align}
Since $\rho_i^\pi$ are probability vectors on $X$, $\|\rho_i^{\pi_r}-\rho_i^{\pi_s}\|_1\le \|\rho_i^{\pi_r}\|_1+\|\rho_i^{\pi_s}\|_1=2$, giving \eqref{eq:perf-core}. Taking $r=r_\star$ and $s=\bar r$ gives \eqref{eq:perf-rbar}, and $\|\,\cdot\,\|_\infty\le\|\,\cdot\,\|_2$ gives the last inequality in \eqref{eq:perf-rbar}. Finally, Theorem~1 (reward-space bound) provides
\begin{align}
\|r_\star-\bar r\|_2
\;\le\; \frac{2\sqrt{2}\,D\,Z(r_\star)}{\delta\,\sqrt{Z_{\min}}}
\Big(\sum_{j=1}^K \alpha_j\,\varepsilon_j\Big)^{\!1/2},
\end{align}
and substituting this into \eqref{eq:perf-rbar} yields \eqref{eq:perf-explicit}.
\end{proof}
Inequality~\eqref{eq:perf-explicit} quantitatively explains the gains in our empirical results, as the barycentric fusion contracts the reward error (Theorem~\ref{thm:fed-irl-ot-main}), the induced policy on each client is guaranteed to achieve a return within $O\!\left((\sum_j \alpha_j\varepsilon_j)^{1/2}\right)$ of the expert, with a prefactor that depends only on $(\gamma, D, \delta, Z(r_\star), Z_{\min})$ and not on the number of clients or their heterogeneity.
Under bounded local IRL error, barycentric fusion produces (i) a global reward and parameter vector that remain close to the population reward, and (ii) a policy whose discounted return on each client’s own dynamics is provably near-optimal under that population reward. 
\begin{table}[t]
    \centering
    \small
    \caption{Average success rates (\%) with mean $\pm$ std. The clients are evaluated on in-distributions (client environments) and held-out (unseen environments). The $20\times20$ case has 20 clients, the $10\times10$ case has 10 clients, $5\times5$ case has 3 clients.}
    \label{tab:avg_success_hier}
    \setlength{\tabcolsep}{5.5pt}
    \begin{tabular}{lcccccc}
        \toprule
        & \multicolumn{3}{c}{\textbf{In-distribution}} & \multicolumn{3}{c}{\textbf{Held-out}} \\
        \cmidrule(lr){2-4} \cmidrule(lr){5-7}
        \textbf{Environment} & \textbf{Local} & \textbf{Mean} & \textbf{Barycenter} & \textbf{Local} & \textbf{Mean} & \textbf{Barycenter} \\
        \midrule
        20$\times$20 & $73.7 \pm 14.0$ & $73.5 \pm 14.5$ & $\mathbf{76.4 \pm 13.4}$ & -- & $75.3 \pm 23.0$ & $\mathbf{81.4 \pm 19.0}$ \\
        10$\times$10 & $80.9 \pm 13.9$ & $81.6 \pm 14.2$ & $\mathbf{89.1 \pm 10.7}$ & -- & $76.1 \pm 8.0$ & $\mathbf{80.3 \pm 9.9}$ \\
        5$\times$5 & $84.2 \pm 12.6$ & $85.1 \pm 11.9$ & $\mathbf{90.5 \pm 9.8}$ & -- & $78.4 \pm 7.3$ & $\mathbf{82.6 \pm 8.1}$ \\
        \bottomrule
    \end{tabular}
\end{table}

\begin{figure}[t]
    \centering
    \includegraphics[width=0.5\linewidth]{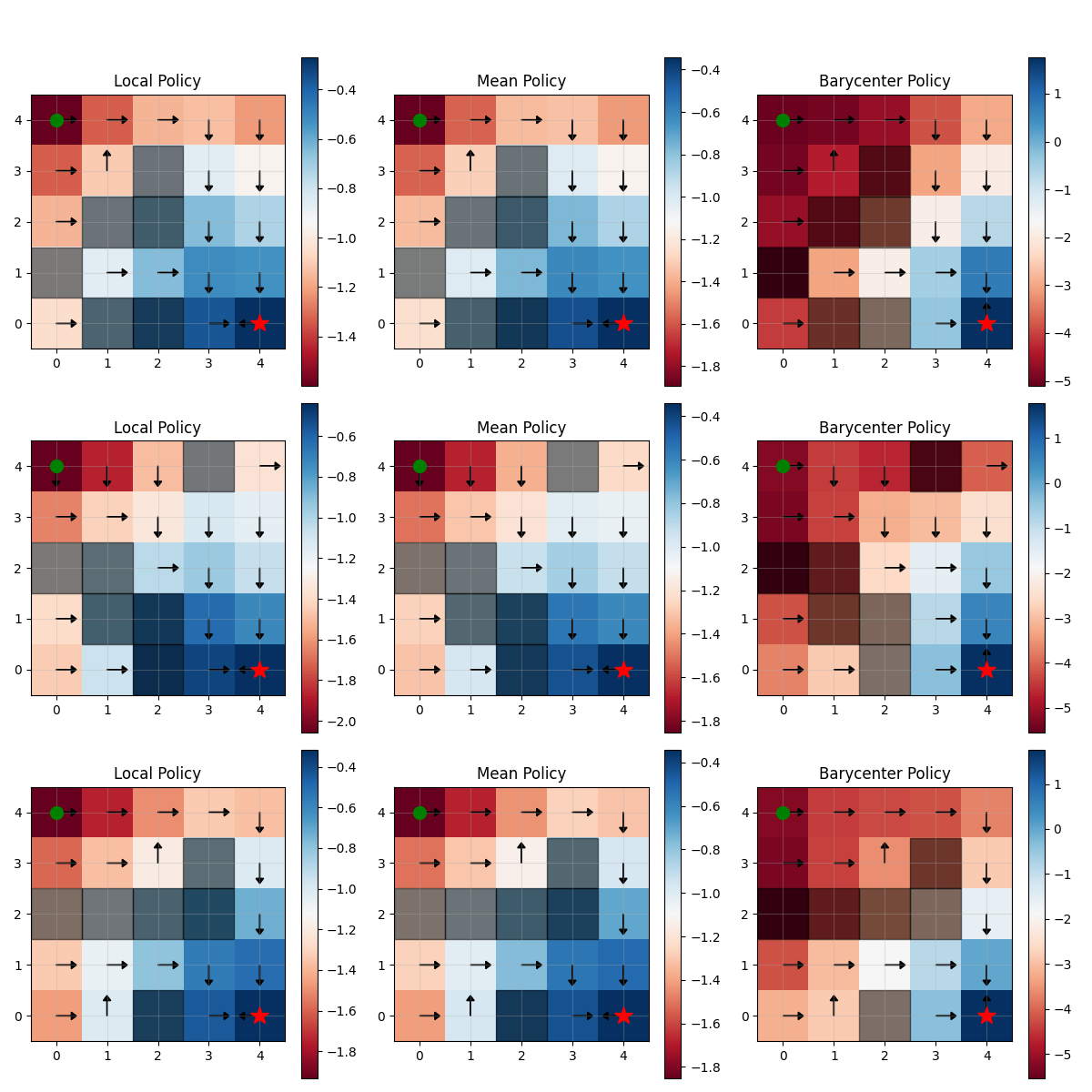}
    \caption{Local vs. mean vs. barycentric rewards for three heterogeneous 5×5 grid-world clients. Local MaxEnt IRL shows spurious obstacle attraction; parameter averaging blurs semantics. Wasserstein barycentric fusion preserves geometry, suppresses artifacts, and recovers a smooth, goal-consistent field across clients.}
    \label{fig:method_comp}
\end{figure}

\section{Empirical Results}\label{sec:empirical}
In this section, we present two case studies. The first provides an in-depth analysis in a discrete grid-world setting, while the second evaluates our method across a set of Gymnasium benchmarks.

\subsection{Grid-World Navigation}
We first evaluate the framework on a stochastic grid-world navigation task. All clients share the same state and action spaces but differ in obstacle layouts and transition stochasticity. The feature representation uses the $\ell_2$ distance to the goal and to the nearest obstacle, with episodes terminating on obstacle collision; each agent’s slip is drawn as $u \sim U(0, 0.10)$. As shown in Table~\ref{tab:avg_success_hier}, barycentric fusion consistently outperforms both unfused local rewards and simple parameter averaging across all grid sizes, with the largest gains in the most heterogeneous $20\times20$ setting. The fused reward is smoother and transfers better, yielding higher success rates in both in-distribution and held-out environments. Fig.~\ref{fig:method_comp} illustrates that while several local learners assign spurious attractive (negative) weights to obstacles due to limited coverage, barycentric fusion recovers a globally consistent reward with a positive coefficient on obstacle distance, encouraging safer navigation.

\begin{figure}[t]
    \centering
    % Small left margin for y-axis label
    \begin{minipage}{0.05\columnwidth}
        \centering
        % Shifted upward by 1cm — adjust as needed
        \raisebox{1cm}{\rotatebox{90}{True Environment Return}}
    \end{minipage}%
    % Right main content: two side-by-side plots
    \begin{minipage}{0.93\columnwidth}
        \centering
        \begin{minipage}[t]{0.49\linewidth}
            \centering
            \includegraphics[width=\linewidth]{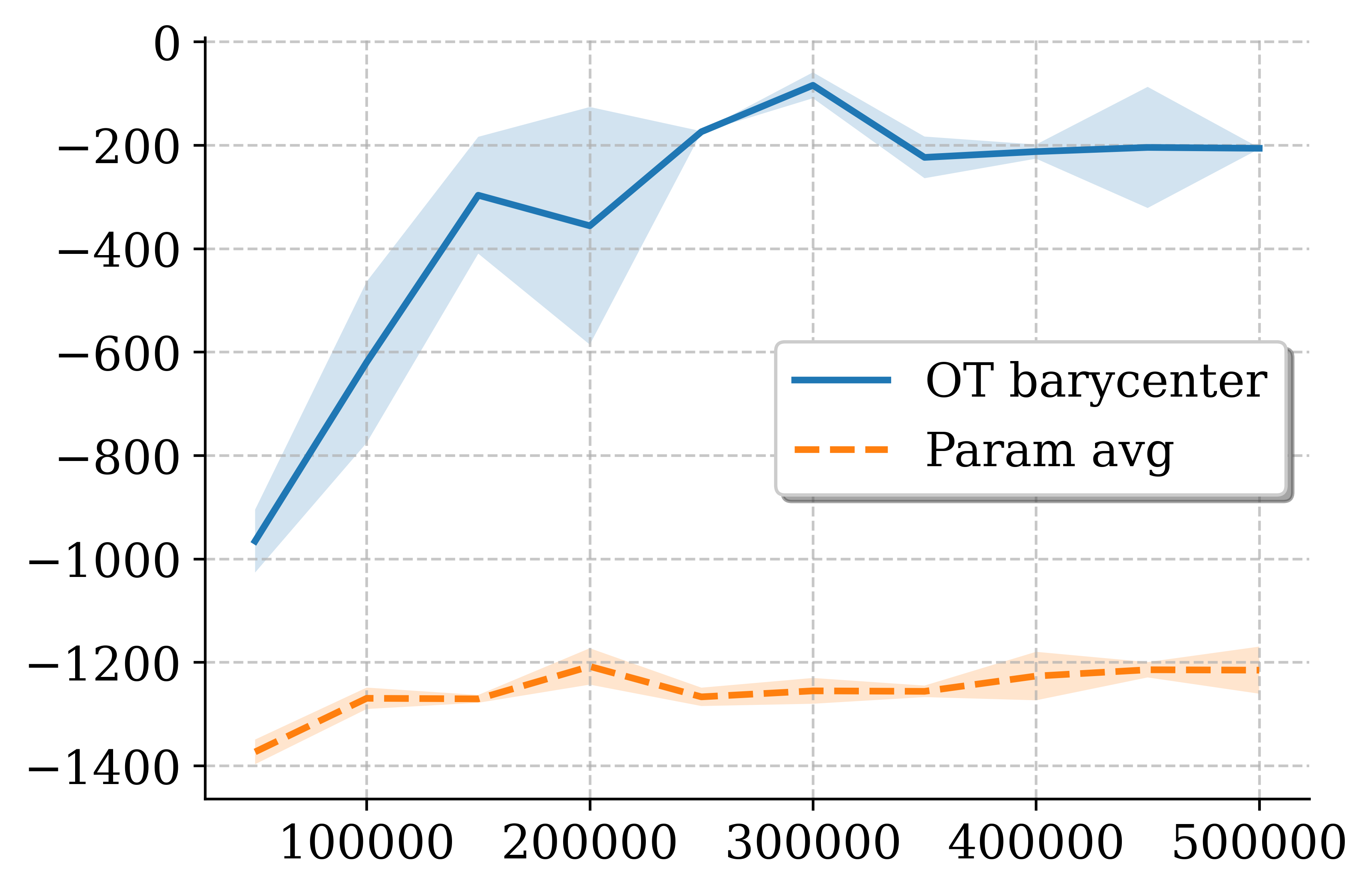}
        \end{minipage}\hfill
        \begin{minipage}[t]{0.49\linewidth}
            \centering
            \includegraphics[width=\linewidth]{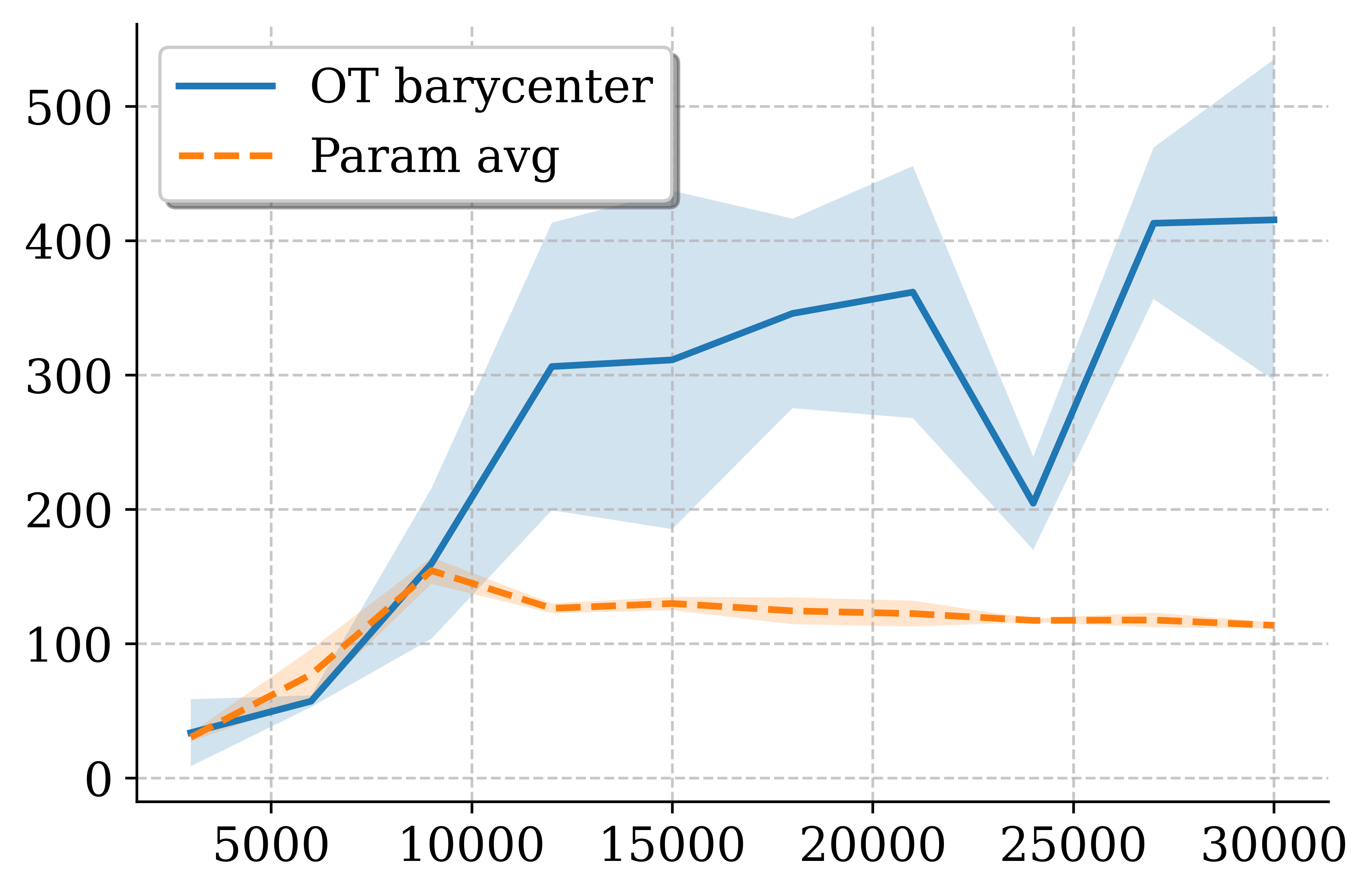}
        \end{minipage}

        % Shared x-axis label
        \vspace{0.3em}
        \parbox{\linewidth}{\centering PPO Training Steps}
    \end{minipage}

    \caption{
    Learning dynamics across environments. The left panel shows results for \texttt{Pendulum-v1} and the right for \texttt{CartPole-v1}. Curves report mean $\pm$~std true environment return versus PPO training timesteps under parameter averaging and Wasserstein barycentric reward fusion.
    }
    \label{fig:learning_curves_all}
\end{figure}

\subsection{Gymnasium Benchmarks}
We next evaluate the barycentric fusion on continuous-control environments from the Gymnasium suite \cite{felten_toolkit_2023}, where each client trains under slightly perturbed physical parameters (e.g., torque limits or gravity) to emulate heterogeneous robotic and environmental dynamics. Unlike the discrete grid-world setting, the continuous state and action spaces here require a feature representation that generalizes smoothly. We adopt radial basis function (RBF) \cite{majdisova2017radial} features centered over a uniform lattice of sampled states and actions. For tractability in computing the Wasserstein barycenter, these continuous spaces are discretized to construct a finite support for the transport cost. 
The learning curves in Fig.~\ref{fig:learning_curves_all} show that barycentric fusion yields higher returns when compared to parameter averaging. 

\subsection{Ablation}
Finally, we run an ablation in the stochastic $5\times 5$ grid-world with $10$ heterogeneous clients to evaluate the claims made by Theorem~\ref{thm:fed-irl-ot-main}, specifically no weak client dominates. For a fraction of weak clients $p \in \{0.00, 0.25, 0.50\}$ we uniformly sample $p\cdot K$ clients and reduce their local MaxEnt IRL iterations by a random multiplicative factor, while the remaining clients use the full budget. From the resulting local parameters we construct a mean-fused model and a Wasserstein-barycentric fused model and evaluate their success rates the same as previously done. As summarized in Table~\ref{tab:weak_client_ablation}, barycentric fusion matches the mean baseline when $p=0$ and degrades more slowly as $p$ increases, maintaining a consistent advantage especially at $p=0.50$, in line with the stability guarantee of Theorem~\ref{thm:fed-irl-ot-main}.
\begin{table}[t]
    \centering
    \small
    \caption{Ablation on the fraction of weak clients $p$ in the stochastic $5\times5$ grid-world. Reported is mean $\pm$ std success rate (\%) over three seeds.}
    \label{tab:weak_client_ablation}
    \setlength{\tabcolsep}{5.5pt}
    \begin{tabular}{lcccc}
        \toprule
        & \multicolumn{2}{c}{\textbf{In-distribution}} & \multicolumn{2}{c}{\textbf{Held-out}} \\
        \cmidrule(lr){2-3} \cmidrule(lr){4-5}
        $p$ & \textbf{Mean} & \textbf{Barycenter} & \textbf{Mean} & \textbf{Barycenter} \\
        \midrule
        0.00 & $81.97 \pm 3.10$ & $\mathbf{81.97 \pm 1.97}$ & $86.43 \pm 1.86$ & $\mathbf{88.27 \pm 2.21}$ \\
        0.25 & $84.43 \pm 3.35$ & $\mathbf{86.20 \pm 3.15}$ & $84.83 \pm 2.76$ & $\mathbf{85.50 \pm 3.02}$ \\
        0.50 & $80.17 \pm 7.94$ & $\mathbf{85.27 \pm 2.76}$ & $82.13 \pm 0.91$ & $\mathbf{85.67 \pm 2.73}$ \\
        \bottomrule
    \end{tabular}
\end{table}

\section{Discussion}\label{sec:discussion}
The presented results show that the barycentric reward fusion offers a geometry-aware alternative to conventional parameter averaging in federated IRL with many advantages. Across both discrete and continuous domains, the Wasserstein barycenter provides a consistent aggregation mechanism that captures shared behavioral semantics while mitigating the effect of undertrained or noisy local models. This property is particularly valuable in heterogeneous settings where direct parameter averaging may degrade due to differing feature scales, local biases, or insufficient demonstrations. Although the method shows promise, it also has a significant weaknesses. As discussed in the technical appendix, scalability is a central issue. The memory complexity of the discrete entropic Wasserstein barycenter scales as $\mathcal{O}(n^2)$ in the number $n$ of support points, since one must store and operate on an $n\times n$ cost matrix. Consequently, the per-iteration runtime of Sinkhorn-based barycenter solvers is also quadratic in $n$ (and linear in the number of client distributions), which quickly becomes prohibitive as the discretization is refined. In several environments, the state--action space was large enough that a reasonably fine discretization would require prohibitively many support points. In these scenarios we were forced to adopt coarse supports, and under such coarse discretization the method struggled and performed noticeably worse than simple parameter averaging.

\section{Conclusion}
This paper introduced a federated IRL framework that fuses local MaxEnt IRL rewards via a Wasserstein barycenter rather than by parameter averaging. Theoretically, the barycentric fusion is stable under bounded local estimation error, and empirically it yields more interpretable and transferable rewards across heterogeneous clients. While scalability and coverage resolution remain open challenges, the results demonstrate that geometry-aware reward aggregation can serve as a foundation for reliable, privacy-preserving collaboration among resource-limited agents. Future extensions will explore continuous barycenter formulations and personalized fine-tuning for large-scale, real-world robotic systems.

\bibliography{l4dc2026-sample}

\newpage
\appendix

\section{Setups}

In this section we detail each gymnasium setup that we evaluate in our empirical results sections.

\subsection{Pendulum-v1}
For \texttt{Pendulum-v1}, each client is assigned a continuous single-link pendulum with slightly different physical parameters. Client $i$ uses a scaled torque limit, damping coefficient, and gravitational constant,
with values increasing linearly in $i$, together with small Gaussian process noise added to the observations. Concretely, we set a torque scale of $1.0 + 0.1 \cdot i$, a damping scale of $1.0 + 0.05 \cdot i$, gravity $9.81 + 0.2 \cdot i$, and observation noise with standard deviation $0.005 \cdot i$. This produces mild heterogeneity in the ease of control and sensitivity to actions while preserving the same stabilization objective. For IRL and barycentric fusion, we discretize the $(\theta,\dot{\theta},a)$ space using \texttt{--n-theta 32 --n-thdot 21 --n-action 5}.

\subsection{CartPole-v1}
For \texttt{CartPole-v1}, each client operates a cart–pole system with different physical scales. Client $i$ modifies the cart force magnitude, gravity, cart mass, pole mass, and pole length by client-dependent scaling factors. Specifically, we scale the force magnitude by $1.0 + 0.1 \cdot i$, gravity by $1.0 + 0.005 \cdot i$, both masses by $1.0 + 0.05 \cdot i$, and the pole length by $1.0 + 0.05 \cdot i$, and we add Gaussian observation noise with standard deviation $0.005 \cdot i$. These changes alter the difficulty and sensitivity of the balancing task across clients while keeping the reward structure (per-step reward $+1$ until failure) fixed. For the barycenter, we discretize the state--action space using \texttt{--n-x 5 --n-xdot 5 --n-theta 7 --n-thetadot 5 --n-action 2}.

\subsection{LunarLanderContinuous-v3}
For \texttt{LunarLanderContinuous-v3}, we consider a continuous-thrust lunar lander with heterogeneous wind fields. Each client $i$ is instantiated with a different wind power and turbulence level while keeping gravity fixed. We set the wind power parameter to $10.0 + 1.0 \cdot i$ and the turbulence power to $1.0 + 0.2 \cdot i$, which induces increasingly challenging lateral disturbances for higher-index clients. The reward function is the standard shaping reward from the Gymnasium implementation. For barycentric fusion, we construct a product grid over the key state components (position and velocity) using \texttt{--n-ll-x 5 --n-ll-y 5 --n-ll-vx 5 --n-ll-vy 5}, along with a coarse discretization of angle, angular velocity, and leg-contact indicators.

\subsection{Acrobot-v1}
For \texttt{Acrobot-v1}, each client controls an underactuated two-link pendulum with client-specific link masses, lengths, and torque scales. Client $i$ scales both link masses and lengths by $1.0 + 0.05 \cdot i$, which in turn affects the center-of-mass positions and overall inertia. We also scale the available torques by a factor derived from the combined mass and length scaling, leading to different effective actuation strengths across clients. As in the other tasks, small Gaussian observation noise with standard deviation $0.005 \cdot i$ is added to the state. The reward remains the standard sparse terminal reward (successful swing-up). For the Wasserstein barycenter, we discretize the state--action space over the joint angles, angular velocities, and discrete actions using \texttt{--n-theta1 9 --n-theta2 9 --n-thdot1 5 --n-thdot2 5 --n-action 3}.

\section{Hyperparmeters}
This section documents the main hyperparameters used for expert training, IRL optimization, aggregation, and evaluation. Table~\ref{tab:hp-general} details each environments hyperparameters while Table~\ref{tab:hp-fusion-bary} details the hyperparameters specific to our formulation. The parameter average fusion does not have any hyperparameters.

\begin{table}[!htbp]
    \centering
    \caption{General training hyperparameters by stage and environment.}
    \label{tab:hp-general}
    \begin{tabular}{lcccc}
        \toprule
        Hyperparameter & Pend. & Cart. & LLC & Acro \\
        \midrule
        \multicolumn{5}{c}{Expert training} \\
        \midrule
        Expert trainer / algorithm          & PPO & PPO & PPO & PPO \\
        Total expert steps                  & 50000 & 200000 & 300000 & 5000 \\
        Expert batch size                   & 64 & 64 & 64 & 64 \\
        Expert learning rate                & $3\cdot10^{-3}$& $3\cdot10^{-3}$& $3\cdot10^{-3}$&$3\cdot10^{-3}$ \\
        Discount factor $\gamma$            &0.99 & 0.99& 0.99&0.99 \\
        GAE $\lambda$             & 0.95& 0.95& 0.95&0.95 \\
        Clip range            & 0.2& 0.2& 0.2& 0.2\\
        Number of clients                   & 6& 6& 6& 6\\
        \midrule
        \multicolumn{5}{c}{Trajectory generation} \\
        \midrule
        Episodes per client                 & 50 &50 & 50&50 \\
        Max episode length                  & 200 & 500 & 1000 & 500  \\
        \midrule
        \multicolumn{5}{c}{IRL} \\
        \midrule
        IRL algorithm                       & MaxEnt& MaxEnt& MaxEnt& MaxEnt\\
        IRL iterations / steps              & 50 & 25 & 100 & 25 \\
        IRL learning rate                   & $5\cdot10^{-3}$& $5\cdot10^{-3}$& $5\cdot10^{-3}$& $5\cdot10^{-3}$\\
        Entropy weight / temperature        & $1\cdot10^{-3}$& $1\cdot10^{-3}$& $1\cdot10^{-3}$& $1\cdot10^{-3}$\\
        Feature map / dimension             & 64& 64& 64& 64\\
        Reward regularization (e.g., L2)    & $l_2$& $l_2$& $l_2$& $l_2$\\
        \midrule
        \multicolumn{5}{c}{Evaluation} \\
        \midrule
        Evaluation trainer / algorithm      & PPO& PPO& PPO&PPO \\
        Evaluation steps / episodes/client  & 500000& 30000 & 2000000&25000 \\
        Number of seeds                     & 10& 10&10 &10 \\
        \bottomrule
    \end{tabular}
\end{table}
\begin{table}[!htbp]
    \centering
    \caption{Wasserstein barycentric fusion hyperparameters for the proposed method.}
    \label{tab:hp-fusion-bary}
    \begin{tabular}{lcccc}
        \toprule
        Hyperparameter & Pend. & Cart. & LLC & Acro \\
        \midrule
        \multicolumn{5}{c}{Support discretization} \\
        \midrule
        State grid resolution               & $32\times21$ & $5^3\times7$ & $5^4\times1^3$ & $9^2\times5^2$ \\
        Action grid resolution              & 5 & 2 & 3 & 5 \\
        Total number of support points $n$  & 3360 & 1750 & 1875 & 10125 \\
        \midrule
        \multicolumn{5}{c}{OT barycenter solver} \\
        \midrule
        Entropic regularization $\varepsilon$    & 0.5& 0.5& 0.5& 0.5\\
        Sinkhorn iterations per barycenter step  & 10& 10& 10& 10\\
        Number of barycenter iterations          & 50& 50& 50& 50\\
        \bottomrule
    \end{tabular}
\end{table}

\section{Limitations}
Figure~\ref{fig:learning_curves_all_limit} illustrates two continuous-control environments \texttt{LunarLanderContinuous-v3} and \texttt{Acrobot-v1}—where the benefits of barycentric fusion begin to diminish. Such coarse discretization was necessary to maintain tractability but inherently limits the precision of the recovered reward and the smoothness of the resulting policy gradients. This exposes two primary limitations of the current framework. First, scalability: the computational cost of computing Wasserstein barycenters grows rapidly with the dimensionality and resolution of the discretized support, as both memory and runtime scale quadratically with the number of bins, even under entropic regularization. Second, resolution on underexplored states: when expert demonstrations offer limited coverage, the fused reward tends to oversmooth or regress toward the entropic prior, reducing fidelity in low-density regions of the state space.
\begin{figure}[t]
    \centering
    \begin{minipage}{0.05\columnwidth}
        \centering
        \raisebox{1cm}{\rotatebox{90}{True Environment Return}}
    \end{minipage}%
    \begin{minipage}{0.93\columnwidth}
        \centering
        \begin{minipage}[t]{0.49\linewidth}
            \centering
            \includegraphics[width=\linewidth]{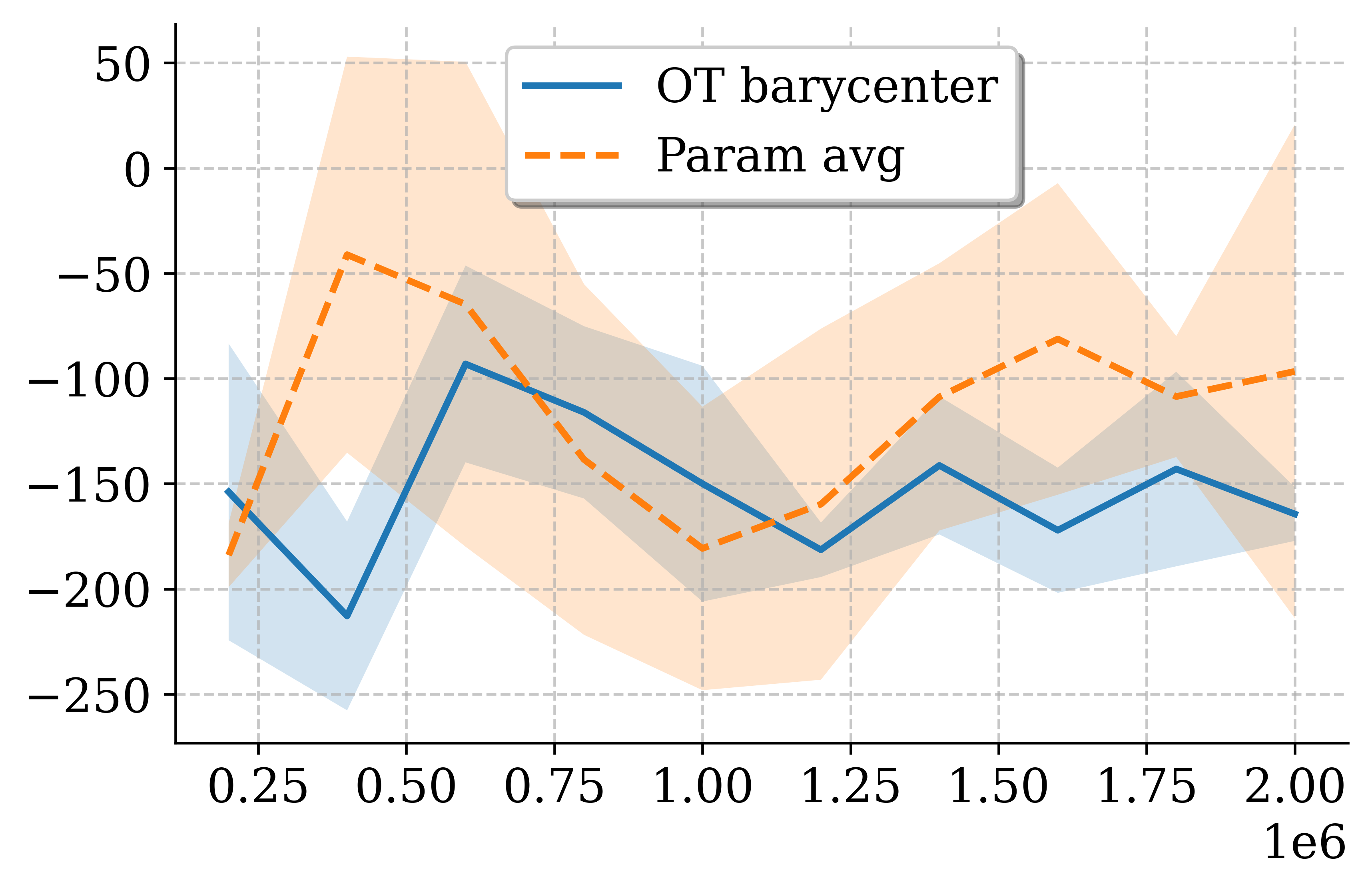}
        \end{minipage}\hfill
        \begin{minipage}[t]{0.49\linewidth}
            \centering
            \includegraphics[width=\linewidth]{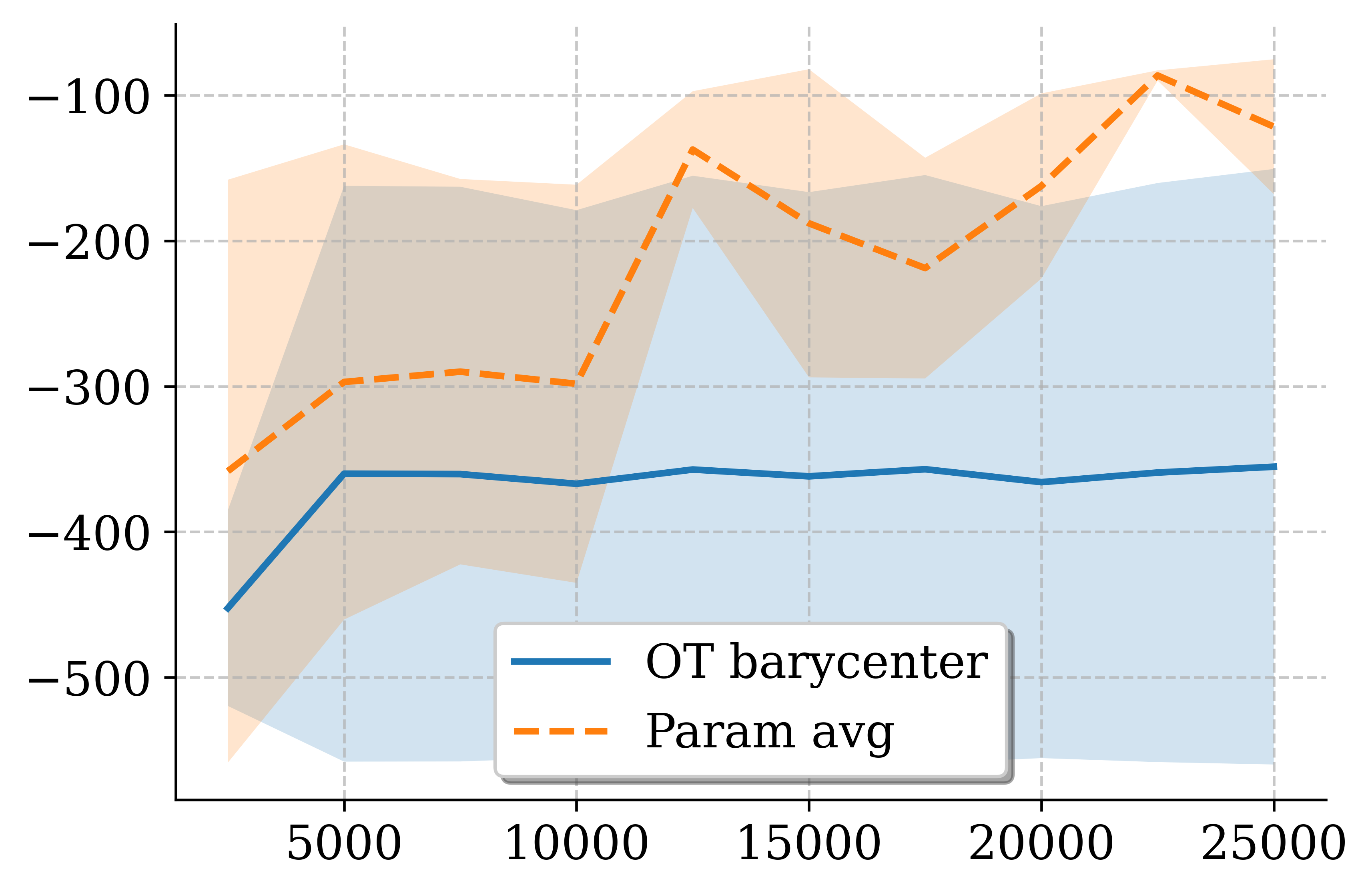}
        \end{minipage}
        \vspace{0.3em}
        \parbox{\linewidth}{\centering PPO Training Steps}
    \end{minipage}
    \caption{
    Learning dynamics across higher-dimensional environments.
    Mean $\pm$~std true environment return versus PPO training timesteps
    under parameter averaging and Wasserstein barycentric reward fusion.
    These cases required coarse discretization due to the high dimensionality of the support.
    }
    \label{fig:learning_curves_all_limit}
\end{figure}

\section{Code Availability}
Our repository is available under the Apache License version 2.0. All checkpoints, output, and data can be provided as a request to the corresponding author.

\end{document}